\documentclass{article} 
\usepackage[utf8]{inputenc}
\usepackage[height=8.4in,width=6.2in]{geometry}

\pdfoutput=1                    
\usepackage{floatrow} 
\newfloatcommand{capbtabbox}{table}[][\FBwidth]

\usepackage[table]{xcolor}
\usepackage{graphicx}
\usepackage{amsmath,amssymb,amsthm} 
\usepackage{amsfonts}
\usepackage{bm}
\usepackage{xspace}
\usepackage{array}
\usepackage{enumerate}
\usepackage{algorithm}
\usepackage{algpseudocode}
\usepackage{appendix}
\usepackage{subfigure}
\usepackage[sort&compress,numbers]{natbib}
\usepackage[colorlinks=true,citecolor=blue,breaklinks]{hyperref}
\usepackage{paralist}
\usepackage{wrapfig}
\usepackage{fancyhdr}

\newcommand{\vm}{\bm{m}}

\newcommand{\vt}{\bm{t}}

\newcommand{\vx}{\bm{x}}                       
\newcommand{\vy}{\bm{y}}

\newcommand{\valpha}  {\bm{\alpha}}

\newcommand{\vmu}     {\bm{\mu}}

\newcommand{\msigma}  {\bm{\Sigma}}

\newcommand{\mr}{\bm{R}}       
\newcommand{\ms}{\bm{S}}       
       
\newcommand{\mU}{\bm{U}}

    \newcommand{\Ac}{\mathcal{A}}

    \newcommand{\Lc}{\mathcal{L}}   
  \newcommand{\Mc}{\mathcal{M}}   
    \newcommand{\Nc}{\mathcal{N}}

\newcommand{\norm}[1]{\|{#1}\|}

\newcommand{\nlsum}{\sum\nolimits}

\newcommand{\set}[1]{\{ #1\}}

\newcommand{\reals}{\mathbb{R}}

\newcommand{\pp}{\mathbb{P}}

\newcommand{\half}{\tfrac{1}{2}}

\newcommand{\xqedhere}[2]{%
  \rlap{\hbox to#1{\hfil\llap{\ensuremath{#2}}}}}

\DeclareMathOperator{\trace}{tr}

\makeatletter
\newlength\aftertitskip     \newlength\beforetitskip
\newlength\interauthorskip  \newlength\aftermaketitskip

\def\maketitle{\par
 \begingroup
   \def\thefootnote{\fnsymbol{footnote}}
   \def\@makefnmark{\hbox to 0pt{$^{\@thefnmark}$\hss}}
   \@maketitle \@thanks
 \endgroup
\setcounter{footnote}{0}
 \let\maketitle\relax \let\@maketitle\relax
 \gdef\@thanks{}\gdef\@author{}\gdef\@title{}\let\thanks\relax}

\def\@startauthor{\noindent \normalsize\bf}
\def\@endauthor{}
\def\@starteditor{\noindent \small {\bf Editor:~}}
\def\@endeditor{\normalsize}
\def\@maketitle{\vbox{\hsize\textwidth
 \linewidth\hsize \vskip \beforetitskip
 {\begin{center} \LARGE\@title \par \end{center}} \vskip \aftertitskip
 {\def\and{\unskip\enspace{\rm and}\enspace}%
  \def\addr{\small}%
  \def\email{\hfill\small\sf}%
  \def\name{\normalsize\bf}%
  \def\AND{\@endauthor\rm\hss \vskip \interauthorskip \@startauthor}
  \@startauthor \@author \@endauthor}
}}

\makeatother

\pdfoutput=1                    

\newtheorem{theorem}{Theorem}

\newtheorem{lemma}{Lemma}
\newtheorem{prop}{Proposition}

\numberwithin{equation}{section}
\numberwithin{theorem}{section}

\newcommand{\mitu}{Laboratory for Information and Decision Systems\\ Massachusetts Institute of Technology, Cambridge, MA.}
\newcommand{\teh}{School of ECE, College of Engineering, University of Tehran, Tehran, Iran}

\begin{document}

\title{Manifold Optimization for Gaussian Mixture Models}
\author{\name Reshad Hosseini  \email reshad.hosseini@ut.ac.ir\\
  \addr \teh
  \AND 
  \name Suvrit Sra\email suvrit@mit.edu\\
  \addr \mitu
}

\maketitle

\vskip0.4cm
\hrule
\vskip0.4cm

\pagestyle{fancy}
\thispagestyle{empty}
\begin{abstract}
  We take a new look at parameter estimation for Gaussian Mixture Models (GMMs). In particular, we propose using \emph{Riemannian manifold optimization} as a powerful counterpart to Expectation Maximization (EM). An out-of-the-box invocation of manifold optimization, however, fails spectacularly: it converges to the same solution but vastly slower. Driven by intuition from manifold convexity, we then propose a reparamerization that has remarkable empirical consequences. It makes manifold optimization not only match EM---a highly encouraging result in itself given the poor record nonlinear programming methods have had against EM so far---but also outperform EM in many practical settings, while displaying much less variability in running times. We further highlight the strengths of manifold optimization by developing a somewhat tuned manifold LBFGS method that proves even more competitive and reliable than existing manifold optimization tools.
We hope that our results encourage a wider consideration of manifold optimization for parameter estimation problems.
\end{abstract}

\section{Introduction}
Gaussian Mixture Models (GMMs) are widely used in a variety of areas, including machine learning and signal processing~\citep{dudahart,keener,bishop,murphy12,McLPee00}. A quick search of the literature suggests that for estimating parameters of a GMM the Expectation Maximization (EM) algorithm~\citep{dempster77} is a \emph{de facto} choice. Although other numerical approaches have also been considered~\citep{redWal84}, methods such as conjugate gradients, quasi-Newton, Newton, are typically inferior to EM~\citep{jordan96} in many practical settings.

The main difficulty of applying standard nonlinear programming techniques for GMMs is optimization over covariance matrices. The positive definiteness constraint, although an open subset of Euclidean space, can be difficult to handle, especially for higher-dimensional problems. When approaching the boundary of the constraint set, convergence speed of iterative methods can also get adversely affected.  A partial remedy for these difficulties is to use the Cholesky decomposition, as was also exploited for semidefinite programming in~\citep{burer1999solving}. But as pointed out in~\citep{vanderbei2000formulating}, for general optimization problems (even for semidefinite programs) such a nonconvex decomposition adds many more stationary points and possibly spurious local minima. One can formulate the positive definiteness constraint via a set of smooth convex inequalities~\citep{vanderbei2000formulating} and resort to interior-point methods. It was observed in~\citep{sra2013geometric} that using such sophisticated methods can be extremely slower (on a class of statistical problems) than simpler EM-like fixed point iterations, especially for higher dimensions. 

In this paper we reconsider the above viewpoint and take a new look at nonlinear optimization techniques for GMM parameter estimation, which can not only match EM but often also outdo it. We believe that matching EM's performance on nontrivial GMMs using such numerical methods is already remarkable. Even more interesting are instances where we substantially outperform EM. 

Specifically, we approach GMM parameter estimation via \emph{Riemannian Manifold Optimization}. We turn to manifold optimization motivated by a simple observation: the positive definiteness constraint on covariance matrices poses difficulties to all numerical methods (gradient-descent, conjugate gradients, quasi-Newton, etc.); and one way to ameliorate these difficulties is by operating directly on the manifold of positive definite matrices.\footnote{Equivalently, on the interior of the constraint set, as is done by interior point methods (their nonconvex versions); though these turn out to be slow too as they are second order methods.}. Therewith, one implicitly satisfies the constraints, and can devote greater effort to the maximization of the log-likelihood.

A reader familiar with the simplicity and elegance of EM may question the above motivation. And this skepticism is justified: an out-of-the-box invocation of manifold optimization turns out to be vastly inferior to EM. So, should we discard manifold optimization too? \emph{No.} But we do need to develop a more refined approach; we outline our ideas below.

Intuitively, the mismatch lies in the geometry. Recall that for GMMs, the M-step of EM is a Euclidean convex optimization problem (which even has a closed form solution), whereas the log-likelihood is not manifold convex\footnote{That is, convex along geodesic curves on a manifold.} even for a single Gaussian. This  suggests that it may be fruitful to consider a reparametrization which makes at least the single component log-likelihood manifold convex. This intuition turns out to have remarkable empirical consequences (Fig.~\ref{fig:reparam}), which ultimately enables manifold optimization to compete with EM and often even surpass it.

\textbf{Contributions.} In light of the above background, the main contributions of this paper are as follows:
\begin{list}{--}{\leftmargin=2em}
\setlength{\itemsep}{-1pt}
\item Introduction of manifold optimization as a powerful numerical tool for GMM parameter estimation. Most importantly, we show how a simple reparamerization holds the key to making manifold optimization succeed.
\item Development of a solver based on manifold-LBFGS; our key contribution here is the design and implementation of a powerful line-search procedure. This line-search helps ensure convergence, and beyond that, it helps LBFGS outperform both EM and the usual manifold conjugate gradient (CG) method; our solver may thus also be of independent interest.
\item Experimental evidence on both synthetic and real-data to show a performance comparison between manifold optimization and EM.
\end{list}
As may be gleaned from our results, manifold optimization performs well across a wide range of parameter values and problem sizes, while being much less sensitive to overlapping data than EM, and displaying less variability in running times. These results are encouraging and suggest that manifold optimization could open a new  algorithmic avenues for handling mixture models. 

We would like to note that for ensuring reproducibility of our results and as a service to the community, we will release our \textsc{Matlab} implementation of the methods developed in this paper. The manifold CG method that we use is directly based on the excellent toolkit \textsc{ManOpt}~\citep{boumal2014manopt}.

\paragraph{Related work.} The published work on EM is huge, so a summary is impossible. Instead, let us briefly mention a few lines of related work. \citet{jordan96} examine several aspects of EM for GMMs and counter the claims of~\citet{redWal84}, who thought EM to be inferior to general purpose nonlinear programming techniques, especially second-order methods. However, it is well-known, see e.g.,~\citep{jordan96,redWal84}, that EM can attain good likelihood values rapidly, and it scales to much larger problems than amenable to second-order methods. Local convergence analysis of EM is available in~\citep{jordan96}, with more refined and precise results in~\citep{ma2000asymptotic}, who formally show that when data have low overlap, EM can converge locally superlinearly. Our paper develops manifold LBFGS, which being a quasi-Newton method can also display local superlinear convergence.

For GMMs some innovative gradient-based methods have also been suggested~\citep{naim2012convergence,salakhutdinov2003optimization}. In order to satisfy positive definite constraint, the authors suggest to use Cholesky decomposition of covariance matrices. Such a reparametrization makes the objective function of even a single Gaussian nonconvex, and adds spurious stationary points to the objective function. Also, these works report results only for low-dimensional problems and spherical (near spherical) covariance matrices. 

The idea of manifold optimization is new for GMM, but in itself it is a well-developed branch of nonlinear optimization. A classic reference is~\citep{udriste}; a more recent work is~\citep{absil2009optimization}; and even a \textsc{Matlab} toolbox exists now~\citep{boumal2014manopt}. In machine learning, manifold optimization has witnessed increasing interest\footnote{Manifold optimization should not be confused with ``manifold learning'' a separate problem altogether.}, e.g., for low-rank optimization~\citep{vandereycken2013low,journee2010low}, or optimization based on geodesic convexity~\citep{sra2013geometric,wiesel12}. 

Beyond numerics, there is substantial interest in theoretical analysis of mixture models~\citep{dasgupta1999learning,moitra2010,kakade15,balakrishnan2014statistical}. These studies are of great theoretical value (though sometimes limited to either low-dimensional, or small number of mixture components, or spherical Gaussians, etc.), but are orthogonal to our work which focuses on highly practical algorithms for general GMMs.


\section{Background and problem setup}

We begin with some background material, which also serves to establish notation. The key quantity in this paper is the \emph{Gaussian Mixture Model (GMM)} for vectors $\vx \in \reals^d$:
\begin{equation*}
  p(\vx) := \nlsum_{j=1}^K\alpha_j p_{\Nc}(\vx; \vmu_j, \msigma_j),
\end{equation*}
where $p_{\Nc}$ is a (multivariate) Gaussian density with mean $\vmu \in \reals^d$ and covariance $\msigma \succ 0$, i.e.,
\begin{equation*}
  p_{\mathcal{N}}(\vx;\vmu,\msigma) := \det(\msigma)^{-1/2}(2\pi)^{-d/2} \exp \bigl( -\tfrac12(\vx-\vmu)^T \msigma^{-1} (\vx-\vmu) \bigr).
\end{equation*}
Given i.i.d.\ samples $\set{\vx_1,\ldots,\vx_n}$, we seek to estimate  $\set{\hat{\vmu}_j \in \reals^d, \hat{\msigma}_j \succ 0}_{j=1}^K$ and $\hat{\valpha} \in \Delta_K$, the $K$-dimensional probability simplex, via maximum likelihood estimation. This task requires solving the \emph{GMM optimization problem}:
\begin{equation}
  \label{eq:2}
  \max_{\valpha \in \Delta_K,\set{\vmu_j,\msigma_j \succ 0}_{j=1}^K}\quad
  \sum_{i=1}^n\log\Bigl(\nlsum_{j=1}^K\alpha_j p_{\Nc}(\vx_i; \vmu_j,\msigma_j)\Bigr).
\end{equation}
Problem~\eqref{eq:2} in general can require exponential time~\citep{moitra2010}.\footnote{Though recent work shows that under strong assumptions, it has polynomial smoothed complexity~\citep{kakade15}.}  However, our focus is more pragmatic: similar to EM, we also seek to efficiently compute local solutions. Our methods are set in the framework of manifold optimization~\citep{absil2009optimization,udriste}; so let us now recall some material on manifolds.

\subsection{Manifolds and geodesic convexity}
\label{sec:manifold}

A smooth manifold is a non-Euclidean space that locally resembles Euclidean space~\citep{lee12}. For optimization, it is more convenient to consider Riemannian manifolds (smooth manifolds equipped with an inner product on the tangent space at each point). These manifolds possess structure that allows one to extend the usual nonlinear optimization algorithms~\citep{udriste,absil2009optimization} to them. 

Algorithms on manifolds often rely \emph{geodesics}, i.e., curves that (locally) join points along shortest paths. Geodesics help generalize Euclidean convexity to \emph{geodesic convexity}. In particular, say $\Mc$ is a Riemmanian manifold, and $x, y \in \Mc$; also let 
\begin{equation*}
  \gamma_{xy}:[0,1] \to \Mc,\quad \gamma_{xy}(0) = x,\ \gamma_{xy}(1) = y,
\end{equation*}
be a geodesic joining $x$ to $y$. Then, a set $\Ac \subseteq \Mc$ is \emph{geodesically convex} if for all $x, y \in \Ac$ there is a geodesic $\gamma_{xy}$ contained within $\Ac$. Further, a function $f: \Ac \to \reals$ is geodesically convex if for all $x, y \in \Ac$, the composition $f \circ \gamma_{xy}: [0,1] \to \reals$ is convex in the usual sense.

The manifold of interest to us in this paper is $\pp^d$, the manifold of $d\times d$ symmetric positive definite matrices. At any point $\msigma \in \pp^d$, the tangent space is isomorphic to entire set of symmetric matrices; and the Riemannian metric at $\msigma$ is given by $\trace(\msigma^{-1}d\msigma\msigma^{-1}d\msigma)$. This metric induces a geodesic from $\msigma_1$ to $\msigma_2$ that happens to even have a closed-form, specifically~\citep{bhatia07},
\begin{equation*}
  \gamma_{\msigma_1,\msigma_2}(t) := \msigma_1^{1/2}(\msigma_1^{-1/2}\msigma_2\msigma_1^{-1/2})^t\msigma_1^{1/2},\quad 0 \le t \le 1.
\end{equation*}
Thus, a function $f: \pp^d \to \reals$ if geodesically convex on $\pp^d$ if it satisfies
\begin{equation*}
  f(\gamma_{\msigma_1,\msigma_2}(t)) \le (1-t)f(\msigma_1) + tf(\msigma_2),\qquad t \in [0,1],\ \msigma_1,\msigma_2 \in \Ac.
\end{equation*}
Such functions can be nonconvex in the Euclidean sense, but remain globally optimizable due to geodesic convexity. This property has been important in some matrix theoretic applications~\citep{bhatia07,sra15}, and has gained more extensive coverage in several recent works~\citep{ring2012optimization,sra2013geometric,wiesel12}. 

We emphasize that even though the mixture cost~\eqref{eq:2} is not geodesically convex, for GMM optimization geodesic convexity seems to play a crucial role, and it has a huge impact on convergence speed. This behavior is partially expected and analogous to EM, where a convex M-Step makes the overall method much more practical. The next section uses this intuition to elicit geodesic convexity.

\subsection{Problem reformulation}
\label{sec:prob}
We begin with parameter estimation for a single Gaussian: although this has a closed-form solution (which ultimately benefits EM), it requires more subtle handling when applying manifold optimization.  Consider therefore, maximum likelihood parameter estimation for a single Gaussian:
\begin{equation}
  \label{eq:3}
  \max_{\vmu,\msigma \succ 0}\Lc(\vmu,\msigma) := \nlsum_{i=1}^n\log p_{\Nc}(\vx_i; \vmu, \msigma).
\end{equation}
Although~\eqref{eq:3} is convex in the Euclidean sense, it is \emph{not} geodesically convex on its domain $\reals^d\times\pp^d$, which makes it  geometrically not so well-suited to the positive definite matrix manifold. 

To fix this mismatch and turn~\eqref{eq:3} into a geodesically convex problem, we invoke a simple reparamerization\footnote{This reparamerization in itself is probably folklore; its role in GMM optimization is what is crucial here.} that has far-reaching impact. We augment the sample vectors $\vx_i$ by an extra dimension and consider $\vy_i^T = [\vx_i^T\ 1]$; therewith, we transform~\eqref{eq:3} into the problem
\begin{equation}
  \label{eq:5}
  \max_{\ms \succ 0}\ \widehat{\Lc}(\ms) := \nlsum_{i=1}^n \log q_{\Nc}(\vy_i;\ms),
\end{equation}
where we define $q_{\Nc}(\vy_i;\ms) := 2\pi \exp(\tfrac12) p_\mathcal{N}(\vy_i;\ms)$. Prop.~\ref{prop:gc} proves the key property of~\eqref{eq:5}.

\begin{figure}[t]\small
\centering
\subfigure[\small Single Gaussian]{%
  \label{fig:singlecom}%
  \includegraphics[width=.48\textwidth]{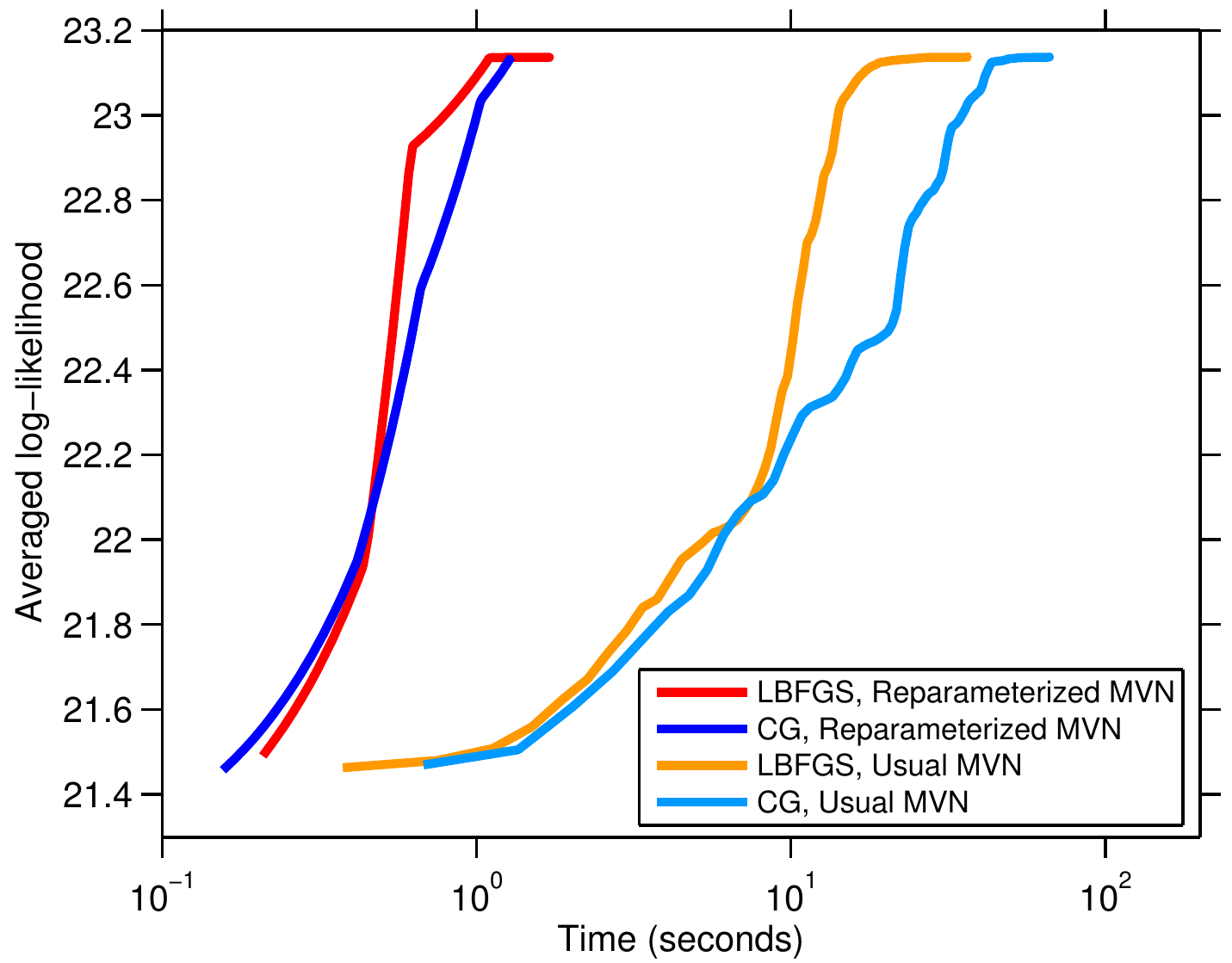}}%
\hfill%
\subfigure[\small Mixtures of seven Gaussians]{%
  \label{fig:sevencom}%
  \includegraphics[width=.47\textwidth]{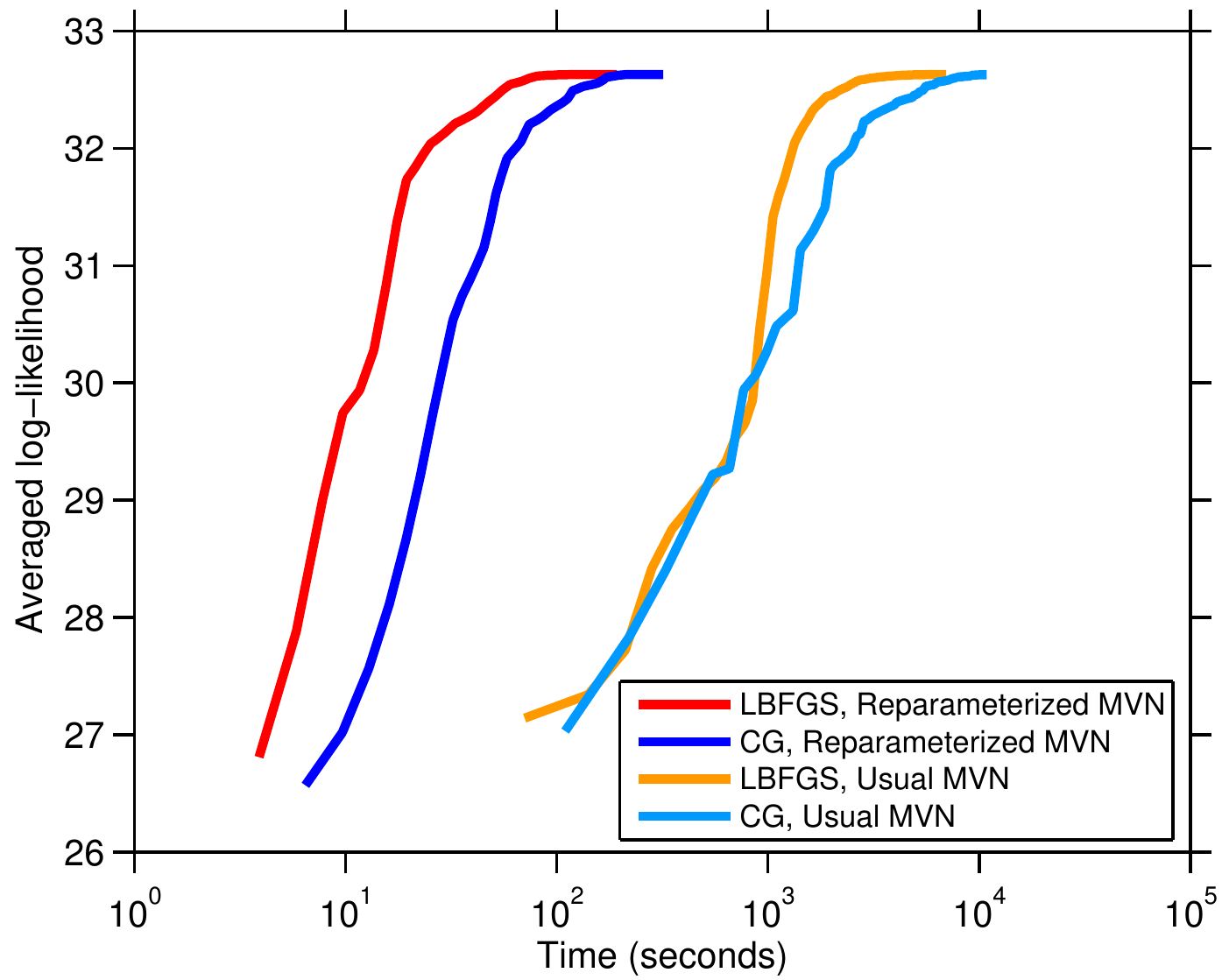}}
\caption{\small
\label{fig:reparam} The effect of reparametrization in convergence speed of manifold CG and manifold LBFGS methods ($d=35$); note that the x-axis (time) is on a logarithmic scale.}
\end{figure}

\begin{prop}
  \label{prop:gc}
  Let $\phi(\ms) \equiv -\widehat{\Lc}(\ms)$, where $\widehat{\Lc}(\ms)$ is as in~\eqref{eq:5}. Then, $\phi$ is geodesically convex.
\end{prop}
We omit the proof for space reasons; it may be found in the appendix.

Theorem~\ref{thm.gauss} shows that solving the reformulation~\eqref{eq:5} also solves the original problem~\eqref{eq:3}.

\begin{theorem} 
  \label{thm.gauss}
  If $\vmu^*, \msigma^*$ maximize~\eqref{eq:3}, and if $\ms^*$ maximizes~\eqref{eq:5}, then $\widehat{\Lc}(\ms^*) = \Lc(\vmu^*, \msigma^*)$ for
  \begin{equation*}
    \ms^* = \begin{pmatrix}
    \msigma^* + \vmu^* {\vmu^*}^T & \vmu^* \\
    {\vmu^*}^T & 1
  \end{pmatrix}.
\end{equation*}
\end{theorem}
\begin{proof}
  We decompose $\ms$ via Schur complements into the components (using \textsc{Matlab} notation):
  \begin{align*}
    \mU = \ms_{\{1:d,1:d\}} - \frac{1}{\ms_{d+1,d+1}}\ms_{\{1:d,d+1\}}\ms_{\{d+1, 1:d\}}, \quad \vt = \ms_{\{1:d,d+1\}},\quad s = \ms_{\{d+1,d+1\}}.
  \end{align*}
  The objective function $\widehat{\Lc}(\ms)$ in terms of these parameters becomes
  \begin{equation*}
    \begin{split}
      \widehat{\Lc}(\mU,\vt,s) &=\text{const} + \tfrac{n}{2} s -\tfrac{n}{2}\det(\mU) - \nlsum_{i=1}^n \tfrac12(\vx_i-\vt)^T \mU^{-1} (\vx_i-\vt) + \tfrac{n}{2s}.
    \end{split}
  \end{equation*}
  Optimizing $\widehat{\Lc}$ over $s>0$ we see that $s^{*}=1$ must hold; so we can eliminate $s$. Hence, the objective reduces to a $d$-dimensional Gaussian log-likelihood, for which clearly $\mU^*=\msigma^*$ and $\vt^*=\vmu^*$.
\end{proof}

Theorem~\ref{thm.gauss} shows that the reparameterization is ``faithful'' as it leaves the optimum unchanged. Figure~\ref{fig:reparam} shows the true import of this reparametrization: its dramatic impact on the empirical behavior Riemmanian Conjugate-Gradient (CG) and Riemannian LBFGS is unmistakable. 

\begin{theorem}
  \label{thm:gmm.reparam}
  A local maximum of the reparameterized GMM log-likelihood
  \begin{equation*}
    \widehat{\Lc}(\{\ms_j\}_{j=1}^K) 
    := \nlsum_{i=1}^n \log\Bigl(\nlsum_{j=1}^K\alpha_j q_{\Nc}(\vy_i; \ms_j)\Bigr)
  \end{equation*}
  is a local minimum of the original log-likelihood
  \begin{equation*}
    \Lc(\{\vmu_j,\msigma_j\}_{j=1}^K) := 
    \nlsum_{i=1}^n \log\Bigl(\nlsum_{j=1}^K\alpha_jp_{\Nc}(\vx_i|\vmu_j,\msigma_j)\Bigr).
  \end{equation*}
\end{theorem}

Theorem~\ref{thm:gmm.reparam} shows that we can replace~\eqref{eq:2} by a reparameterized log-likelihood whose local maxima agree with those of~\eqref{eq:2}. Moreover, the individual components of the reparameterized log-likelihood are geodesically convex, which once again has a huge empirical impact (see Figure~\ref{fig:reparam}). 

We also need to replace the constraint $\valpha \in \Delta_K$ to make the problem unconstrained. We do this via a commonly used change of variables~\citep{jordan1994hierarchical}:
\begin{equation*}
\eta_k =  \log \biggl ( \frac{ \alpha_k}{\alpha_K} \biggr),\quad k=1,\hdots,K-1.
\end{equation*}
Assume $\eta_K=0$ to be a constant, then the final optimization problem is given by:
\begin{equation}
  \label{eq:6}
  \max_{\{\ms_j \succ 0\}_{j=1}^K,\{\eta_j\}_{j=1}^{K-1}}  \widehat{\Lc}(\{\ms_j\}_{j=1}^K,\{\eta_j\}_{j=1}^{K-1}) 
  := \sum_{i=1}^n \log\Bigl(\sum_{j=1}^K\frac{\exp(\eta_j)}{\sum_{k=1}^K \exp(\eta_k)} q_{\Nc}(\vy_i; \ms_j)\Bigr)
\end{equation}
We view~\eqref{eq:6} as a manifold optimization problem; specifically, it is an optimization problem on the product manifold $\bigl(\prod_{j=1}^K\pp^d\bigr) \times \reals^{K-1}$. Let us see how to solve it.

\section{Manifold Optimization}
\label{sec:manopt}

A common approach for unconstrained optimization on Euclidean spaces is to iteratively apply the following two steps: (i) find a descent direction; and (ii) perform a line-search to obtain sufficient decrease (to ensure convergence). 

The difference when optimizing on manifolds is that the descent direction is computed on the tangent space. At a point $X$ on the manifold, the tangent space $T_X$ is the approximating vector space (see Fig.~\ref{fig:cg}). Given a descent direction $\xi_X \in T_X$, line-search is performed along a smooth curve on the manifold (red curve in Fig.~\ref{fig:cg}). The derivative of this curve at point $X$ equals the descent direction $\xi_X$. We refer the reader to~\citep{absil2009optimization,udriste} for an in depth introduction to manifold optimization.

Successful large-scale (Euclidean) optimization methods such as conjugate-gradient and LBFGS, combine gradients at the current point with gradients and descent directions from previous points to generate a descent direction at the current point. To adapt such algorithms to manifolds, in addition to defining gradients on manifolds, we also need to define how to transport vectors in a tangent space at one point, to vectors in a different tangent space at another point.

On Riemannian manifolds, the gradient is simply defined as a direction on the tangent space, where the inner-product of the gradient and another direction in the tangent space gives the directional derivative of the function. Formally, if $g_X$ defines the inner product in the tangent space $T_X$, then
\begin{equation*}
Df(X)\xi = g_X(\text{grad} f(X),\xi),\quad\text{for}\ \xi \in T_X.
\end{equation*} 
Given a descent direction in the tangent space, the curve along which we do the line-search can be a geodesic. A map that takes the direction and a step length, and yields a corresponding point on the geodesic is called an exponential map. A Riemannian manifold also comes with a natural way of transporting vectors on geodesics, which is called parallel transport. Intuitively, a parallel transport is a differential map with zero derivative along the geodesics. Algorithm~\ref{alg.opt} sketches a generic manifold optimization algorithm.

\begin{algorithm}[h]
  \caption{\small Sketch of optimization algorithms (CG, LBFGS) on manifold }
  \label{alg.opt}
  \begin{algorithmic}\small
    \State {\bf Given:} Riemannian manifold $\Mc$ with Riemannian metric $g$; parallel transport $\mathcal{T}$ on $\Mc$; exponential map $R$; initial value $X_0$; a smooth function $f$
    \For{$k=0,1,\ldots$}
      \State Obtain a descent direction based on stored information and $\text{grad} f(X_k)$ using defined $g$ and $\mathcal{T}$
      \State Use line-search to find $\alpha$ such that it satisfies appropraite conditions
      \State Calculate $X_{k+1}=R_{X_k}(\alpha \xi_k)$
      \State Based on the memory and need of algorithm store $X_k$, $\text{grad} f(X_k)$ and $\alpha \xi_k$
    \EndFor
    \State \textbf{return} $X_k$
  \end{algorithmic}
\end{algorithm}

Table~\ref{tbl:psdSummary} summarizes the key quantities for the positive definite matrix manifold. Note that a product space of Riemannian manifolds is again a Riemannian manifold with the exponential map, gradient and parallel transport defined as the Cartesian product of individual expressions; the inner product is defined as the sum of inner product of the components in their respective manifolds.

\begin{table}\small
  \caption{\small Summary of Riemannian expressions for PSD matrices}
  \label{tbl:psdSummary}
  \begin{tabular}{|c|l|}
      \hline
    Definition & Expression for PSD matrices\\
    \hline
    Tangent space& Space of symmetric  matrices \\
    Metric between two tangent vectors $\xi,\eta$ at $\Sigma$ & $g_{\Sigma}(\xi,\eta)= \trace (\Sigma^{-1} \xi \Sigma^{-1} \eta)$  \\
    Gradient at $\Sigma$ if Euclidean gradient is $\nabla f(\Sigma)$ & $\text{grad} f(\Sigma) = \tfrac12 \Sigma (\nabla f(X)+\nabla f(X)^T) \Sigma$\\
    Exponential map at point $\Sigma$ in direction $\xi$ & $R_{\Sigma}(\xi)= \Sigma \exp (\Sigma^{-1} \xi )$  \\
    Parallel transport of tangent vector $\xi$ from $\Sigma_1$ to $
    \Sigma_2 $& $\mathcal{T}_{\Sigma_1,\Sigma_2}(\xi)= E \xi E^T,\quad E =(\Sigma_2\Sigma_1^{-1})^{1/2}$\\
    \hline
  \end{tabular}
\end{table}

Different variants of LBFGS can be defined depending where to perform vector transport. We found that the version developed in \cite{sra15} gives the best performance. We implemented this algorithm together with the crucial line-search algorithm satisfying Wolfe conditions, which we now explain.

\subsection{Line-search algorithm satisfying Wolfe conditions}
\begin{wrapfigure}{l}{5.5cm}
  \centering
  \includegraphics[scale=0.5]{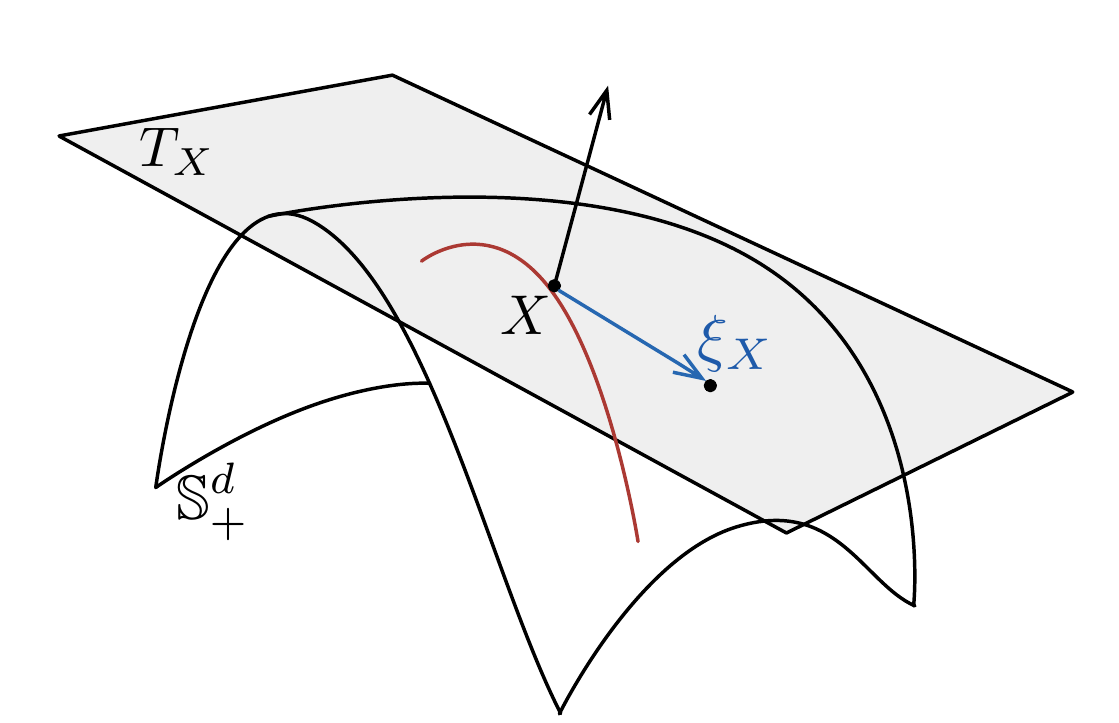}
  \caption{\footnotesize Visualization of line-search on a manifold: $X$ is a point on the manifold, $T_X$ is the tangent space at the point $X$, $\xi_X$ is a descent direction at $X$; the red curve is the curve along which line-search is performed. }
  \label{fig:cg}
\end{wrapfigure}
To ensure LBFGS on the manifold always produces a descent direction, it is necessary to ensure that the line-search algorithm satisfy Wolfe conditions~\citep{ring2012optimization}. 
These conditions are given by:
\begin{align}
f(R_{X_k}(\alpha \xi_k)) &\leq f(X_k) + c_1 \alpha D f(X_k) \xi_k \\
D f(X_{k+1}) \xi_{k+1} &\geq c_2 D f(X_k) \xi_k,
\label{eq.wolfe}
\end{align}
where $0<c_1<c_2<1$. Note that $\alpha D f(X_k) \xi_k=g_{X_{k}}(\text{grad} f(X_k),\alpha \xi_k)$, i.e., the derivative of $f(X_k)$ in the direction $\alpha \xi_k$ equals the inner product of descent direction and gradient of the function. Practical line-search algorithms implement a stronger version of~\eqref{eq.wolfe}, leading to the so-called strong Wolfe condition:
\begin{equation*}
|D f(X_{k+1}) \xi_{k+1}|  \leq c_2D f(X_k) \xi_k.
\end{equation*}

Similar to the line-search algorithm in Euclidean case, the line-search algorithm is divided into two phases: bracketing and zooming~\citep{nocedal2006numerical}. During bracketing, an interval is found such that a point satisfying Wolfe conditions can be found in this interval. In the zooming phase, the actual point in the interval satisfying the conditions is obtained. The one-dimensional function and its gradient that the line-search uses are defined as $\phi(\alpha) = f(R_{X_k}(\alpha \xi_k))$  and $\phi^{\prime}(\alpha) = \alpha D f(X_k) \xi_k$, respectively. The algorithm is the same as the line-search in the Euclidean space, but we present details for its manifold incarnation in the appendix for the reader's convenience. Theory behind how this algorithm is guaranteed to find a step-length satisfying (strong) Wolfe conditions can be found in \citep{nocedal2006numerical}. 

The initial step-length $\alpha_1$ can be guessed using the previous function and gradient information. We propose the following choice that turns out to be quite effective:
\begin{equation}
\label{eq.alpha1}
\alpha_1 = 2\frac{f(X_k) - f(X_{k-1})}{  D f(X_k) \xi_k}.
\end{equation}
Equation~\eqref{eq.alpha1} is obtained by finding $\alpha^*$ that minimizes a quadratic approximation of the function along the geodesic through the previous point (based on $f(X_{k-1})$, $f(X_k)$ and $D f(X_{k-1}) \xi_{k-1}$):
\begin{equation}
\label{eq.quadappr}
\alpha^* = 2\frac{f(X_k) - f(X_{k-1})}{  D f(X_{k-1}) \xi_{k-1}}.
\end{equation}
Then assuming that first-order change will be the same as in the previous step, we write
\begin{equation}
\label{eq.samefirst}
\alpha^*  D f(X_{k-1}) \xi_{k-1} \approx \alpha_1  D f(X_{k}) \xi_{k}.
\end{equation}
Combining~\eqref{eq.quadappr} and~\eqref{eq.samefirst}, we obtain our procedure of selection $\alpha_1$ expressed in~\eqref{eq.alpha1}. \citet{nocedal2006numerical} suggest using either $\alpha^*$ of \eqref{eq.quadappr} for the initial step-length $\alpha_1$, or using~\eqref{eq.samefirst} where $\alpha^*$ is set to be the step-length obtained in the line-search in the previous point. We observed the choice \eqref{eq.alpha1} proposed above, leads to substantially better performance than the other two approaches.

\section{Experimental Results}
\label{sec:expt}

We have performed numerous experiments to examine the effectiveness of the presented method. We report performance comparisons on both real and simulated data. In all experiments, we initialize the mixture parameters using k-means++~\citep{arthur2007}, and we start all methods using the same initialization. All methods also use the same termination criteria: they stop either when the difference of average log-likelihood falls below $10^{-6}$, or when the number of iterations exceed $1500$. Many more results for both simulated data and real data can be found in the appendix. 

\subsection*{Simulated Data}
EM's performance is well-known to depend on the degree of separation of the mixture components~\citep{jordan96,ma2000asymptotic}. To assess the impact of this separation on our methods, we generate data as proposed in~\citep{dasgupta1999learning, verbeek2003efficient}. The distributions are sampled so their means satisfy the following inequality:
\begin{equation*}
\forall_{i \neq j}: \norm{\vm_i - \vm_j}\geq c \max_{i,j} \{ \trace (\msigma_i) - \trace (\msigma_j) \},
\end{equation*}
where $c$ models the degree of separation. Since mixtures with high eccentricity have smaller overlap, in addition to high eccentricity $e=10$ (eccentricity is defined as the ratio of the largest eigenvalue to the smallest eigenvalue of the covariance matrix), we also test the  (spherical) case where components do not have any eccentricity, so $e=1$. We test three levels of separation $c = 0.2$ (low), $c = 1$ (medium) and $c=5$ (high). We test two different numbers of mixture components $K=2$ and $K=5$; we consider experiments with larger values of $K$ for our real data experiments.
\begin{table}\small
  \caption{\small
    Speed and average log-likelihood (ALL) comparisons for $d=20$, $e=10$ (each row reports results averaged over 20 runs over different datasets, so the ALL values are not comparable to each other).}
  \label{tbl:expd20e10}
  \rowcolors{2}{gray!20}{white}
  \begin{tabular}{lr|rr|rr|rr}
    \hline
    & & \multicolumn{2}{c|}{EM Algorithm}   & \multicolumn{2}{c|}{LBFGS Reparametrized} & \multicolumn{2}{c}{CG Reparametrized} \\
    &   & Time (s) & ALL       & Time (s) & ALL& Time (s) & ALL\\
    \hline
 $c=0.2$ & $K=2$ & 1.0 $\pm$ 0.5 & -11.3 & 5.6 $\pm$ 3.2 & -11.3 & 3.6 $\pm$ 1.9 & -11.5 \\
& $K=5$ & 35.4 $\pm$ 53.1 & -12.8 & 50.0 $\pm$ 32.1 & -12.8 & 47.1 $\pm$ 41.9 & -12.9 \\[5pt]
$c=1$ & $K=2$ & 0.5 $\pm$ 0.2 & -10.8 & 3.1 $\pm$ 1.0 & -10.8 & 2.6 $\pm$ 0.7 & -10.8 \\
& $K=5$ & 103.6 $\pm$ 114.6 & -13.5 & 72.9 $\pm$ 62.6 & -13.4 & 42.4 $\pm$ 27.9 & -13.3 \\[5pt]
$c=5$ & $K=2$ & 0.2 $\pm$ 0.2 & -11.2 & 2.9 $\pm$ 1.4 & -11.2 & 2.3 $\pm$ 0.9 & -11.2 \\
& $K=5$ & 36.1 $\pm$ 70.9 & -12.8 & 27.7 $\pm$ 32.5 & -12.8 & 30.4 $\pm$ 42.2 & -12.8 \\
  \end{tabular}
\end{table}

For $e=1$, the results for data with dimensionality equal to 20 are given in Table~\ref{tbl:expd20e10}. The results are obtained after running with 20 different random choices of parameters for each configuration. From the tables it is apparent that the performance of EM  and Riemannian optimization with our reparametrization are very similar. The variance of computation time shown by Riemmanian optimization is, however, notably smaller. 

In another set of simulated data experiments, we apply different algorithms for the case where there is no eccentricity; the results are shown in Table~\ref{tbl:expd20e1}. The interesting case is the case of low separation $c=0.2$, where the condition number of the Hessian becomes large. As predicted by theory, the EM converges very slowly in such a case; Table~\ref{tbl:expd20e1} confirms this claim. It is known that in this such a case, the performance of powerful optimization approaches like CG and LBFGS also degrades~\citep{nocedal2006numerical}. But both CG and LBFGS suffer less than EM, and LBFGS performs noticeably better than CG.

\begin{table}\small
  \caption{\small
Speed and ALL comparisons for $d=20$, $e=1$.}
  \label{tbl:expd20e1}
  \rowcolors{2}{gray!20}{white}
  \begin{tabular}{lr|rr|rr|rr}
    \hline
    & & \multicolumn{2}{c|}{EM Algorithm}   & \multicolumn{2}{c|}{LBFGS Reparametrized} & \multicolumn{2}{c}{CG Reparametrized} \\
    &   & Time (s) & ALL       & Time (s) & ALL& Time (s) & ALL\\
    \hline
    $c=0.2$ & $K=2$ & 72.9 $\pm$ 37.7 & 17.6 & 40.6 $\pm$ 21.6 & 17.6 & 49.4 $\pm$ 31.7 & 17.6 \\
    & $K=5$ & 396.7 $\pm$ 136.6 & 17.5 & 156.1 $\pm$ 80.2 & 17.5 & 216.3 $\pm$ 51.4 & 17.5 \\[5pt]
    $c=1$ & $K=2$ & 7.0 $\pm$ 8.4 & 17.1 & 13.9 $\pm$ 13.7 & 17.0 & 16.7 $\pm$ 18.7 & 17.0 \\
    & $K=5$ & 38.6 $\pm$ 67.0 & 16.2 & 43.8 $\pm$ 38.5 & 16.2 & 58.4 $\pm$ 47.4 & 16.2 \\[5pt]
    $c=5$ & $K=2$ & 0.2 $\pm$ 0.1 & 17.1 & 3.0 $\pm$ 0.5 & 17.1 & 2.7 $\pm$ 0.8 & 17.1 \\
    & $K=5$ & 26.4 $\pm$ 55.3 & 16.1 & 20.2 $\pm$ 18.4 & 16.1 & 23.3 $\pm$ 27.8 & 16.1
  \end{tabular}
\end{table}

\subsection*{Real Data}
We now present performance evaluation on natural image datasets, where mixtures of Gaussians were reported to be a good fit to the data~\citep{zoran2012natural}.  We extracted 200,000 image patches of size $6 \times 6$ from images and subtracted the DC component, leaving us with 35-dimensional vectors. Performance of different algorithms are reported in Table~\ref{tbl:expnatural}. As for simulated results, performance of EM and manifold CG on the reparametrized parameter space is similar. Manifold LBFGS converges notably faster (except for $K=6$) than both EM and CG. Without our reparamerization, performance of the manifold methods degrades substantially; because the experiments take too long to run, we report only the degraded behavior of CG, which runs about 20 times slower than reparametrized CG and LBFGS. Note that for $N=6$ and $N=8$, CG without reparametrization stops because it hits the bound of a maximum 1500 iterations, and therefore its ALL is smaller than the other two methods.

\begin{table}\small
  \caption{\small
Speed and ALL comparisons for natural image data $d=35$.}
  \label{tbl:expnatural}
  \begin{tabular}{lrr|rr|rr|rr}
    \hline
    & \multicolumn{2}{c|}{EM Algorithm}   & \multicolumn{2}{c|}{LBFGS Reparametrized} & \multicolumn{2}{c|}{CG Reparametrized} & \multicolumn{2}{c}{CG Usual} \\
    & Time (s) & ALL       & Time (s) & ALL& Time (s) & ALL & Time (s) & ALL \\
    \hline
    $K=2$ & 16.61 & 29.28 & 14.23 & 29.28 & 17.52 & 29.28 & 947.35 & 29.28 \\
   $K=3$ & 90.54 & 30.95 & 38.29 & 30.95 & 54.37 & 30.95 & 3051.89 & 30.95 \\
   $K=4$ & 165.77 & 31.65 & 106.53 & 31.65 & 153.94 & 31.65 & 6380.01 & 31.64 \\
  $K=5$ & 202.36 & 32.07 & 117.14 & 32.07 & 140.21 & 32.07 & 5262.27 & 32.07 \\
   $K=6$ & 228.80 & 32.36 & 245.74 & 32.35 & 281.32 & 32.35 & 10566.76 & 32.33 \\
    $K=7$ & 365.28 & 32.63 & 192.44 & 32.63 & 318.95 & 32.63& 10844.52 & 32.63 \\
    $K=8$ & 596.01 & 32.81 & 332.85 & 32.81 & 536.94 & 32.81& 14282.80 & 32.58 \\
    $K=9$ & 900.88 & 32.94 & 657.24 & 32.94 & 1449.52 & 32.95 & 15774.88 & 32.77 \\
    $K=10$ & 2159.47 & 33.05 & 658.34 & 33.06 & 1048.00 & 33.06& 17711.87 & 33.03 \\
    \hline 
  \end{tabular}
\end{table}

\section{Conclusions and future work}
We proposed Riemannian manifold optimization as a counterpart to the EM algorithm for fitting Gaussian mixture models. We demonstrated that for enabling manifold optimization to attain its true potential on GMMs, and to either match or outperform EM, it is necessary to represent the parameters in a different space and adjust the cost function accordingly. Extensive experimentation with both experimental and real datasets yielded quite encouraging results, suggesting that manifold optimization may hold the potential to open new algorithmic avenues for mixture modeling. 

Several strands of practical value are immediate from our work (and are a part of our ongoing efforts):  (i) extension to large-scale mixtures (both large $n$ and large $K$) through stochastic manifold optimization~\citep{bonnabel2013}, especially given the importance of stochastic methods in the Euclidean setting; (ii) use of richer classes of priors with GMMs than the usual inverse Wishart priors (which are common, as they leave the M-step simple); this prior is actually geodesic convex and fits within the broader class of geodesic priors that our framework enables; (iii) 
incorporation of penalties for avoiding tiny clusters; such penalties fit in easily in our framework, though they are not as easy to use in the EM framework. 
Moreover, beyond just GMMs, exploration of other mixture models that can benefit from manifold optimization techniques is a fruitful topic worth exploring.

\bibliographystyle{abbrvnat}

\begin{appendices}

\newpage
\section{Technical details}
\label{sec:math}

\subsection{Proof of Proposition~\ref{prop:gc}}
First, we need the following lemma.
\begin{lemma}
  \label{lem:gm}
  Let $\ms$, $\mr \succ 0$. Then, for a vector $\vx$ of appropriate dimension,
  \begin{equation}
    \label{eq:1}
    \vx^T(\ms^{1/2}(\ms^{-1/2}\mr\ms^{-1/2})^{1/2}\ms^{1/2})\vx \le [\vx^T\ms\vx]^{1/2}[\vx^T\mr\vx]^{1/2}.
  \end{equation}
\end{lemma}
\begin{proof}
  Follows from \citep[Thm.~4.1.3]{bhatia07}.
\end{proof}

\begin{proof}{Proof (Prop.~\ref{prop:gc})}
  Since $\phi$ is continuous, it suffices to establish mid-point geodesic convexity:
  \begin{equation*}
    \phi(\gamma_{\ms,\mr}(\half)) 
    \le \half\phi(\ms) + \half\phi(\mr),\qquad\text{for}\ \ms,\mr \in \pp^d.
  \end{equation*}
  Denoting inessential constants by $c$, the above inequality turns into 
  \begin{align*}
    \phi(\gamma_{\ms,\mr}(\half)) 
    &=
    \phi(\ms^{1/2}(\ms^{-1/2}\mr\ms^{-1/2})^{1/2}\mr^{1/2})\\
    &= -\log\det(\ms^{1/2}\mr^{1/2}) + c\nlsum_i\vy_i^T(\ms^{1/2}(\ms^{-1/2}\mr\ms^{-1/2})^{1/2}\ms^{1/2})\vy_i\\
    &\le -\half\log\det(\ms) -\half\log\det(\mr) + c\nlsum_i [\vy_i^T\ms\vy_i]^{1/2}[\vy_i^TR\vy_i]^{1/2}\\
    &\le -\half\log\det(\ms) + \half c\nlsum_i \vy_i^T\ms\vy_i -\half\log\det(\mr) + \half c\nlsum_i \vy_i^T\mr\vy_i\\
    &=\half \phi(\ms) + \half \phi(\mr),
  \end{align*}
  where the first inequality follows from Lemma~\ref{lem:gm}.
\end{proof}

\subsection{Proof of Theorem~\ref{thm:gmm.reparam}}
\begin{proof}
  Let $\ms_1^*, \ldots, \ms_K^*$ be a local maximum of $\widehat{\Lc}$. Then, $\ms_j^*$ is the maximum of the following cost function:
  \begin{equation*}
    \frac12 \nlsum_{i=1}^n w_i \log\det(\ms_j) + \frac12\nlsum_{i=1}^n w_i \vy_i^T \ms_j^{-1} \vy_i,
  \end{equation*} 
  where for each $i \in \set{1,\ldots,n}$ the weight
  \begin{equation*}
    w_i = \frac{q_{\Nc}(\vy_i|\ms^*_j)}{\sum_{j=1}^K\alpha_j q_{\Nc}(\vy_i|\ms^*_j) }.
  \end{equation*}
  Using an argument similar to that for Theorem~\ref{thm.gauss}, we see that $s_j^*=1$, whereby $q_{\Nc}(\vy_i|\ms^*_j) = p_{\Nc}(\vx_i;\vt_j^*,\mU_j^*)$. Thus, at a maximum the distributions agree and the proof is complete.
\end{proof}
\section{Line-search Algorithm}
Algorithm~\ref{alg.wls} summarizes a line-search algorithm satisfying strong Wolfe conditions. The zooming phase of the line-search is given in Algorithm~\ref{alg.zooming}. Like in the Euclidean case $c_1$ is assumed to be a small number, here $10^{-4}$, and $c_2$ is a constant close to one, here $0.9$. For interpolation and extrapolation one can find the minimum of a cubic polynomial approximation to the function in an interval. In each step of interpolation, the interpolation is done on an interval smaller that the actual interval  to have specific distance from end-points of the interval (we used the distance to be 0.1 of the interval length). The interval for the extrapolation is assumed to be between $1.1$ and $10$ times larger than the point we are extrapolating from. For the cubic polynomial interpolation, we use the function $\phi(.)$ and its gradient $\phi^{\prime}(.)$ in the interval. For extrapolation, we use the function and gradient at $0$ and at the end-point.

\begin{minipage}[h]{1.0\linewidth}
  \begin{minipage}[t]{0.5\linewidth}
    \begin{algorithm}[H]
      \caption{\small Line-search satisfying Wolfe conditions }
      \label{alg.wls}
      \begin{algorithmic}[1]\raggedright\small
        \State{\bf Given:} Current point $X_k$ and descent direction $\xi_k$
        \State $\phi(\alpha) \gets f(R_{X_k}(\alpha \xi_k))$;  $\phi^{\prime}(\alpha) \gets \alpha D f(X_k) \xi_k$
        \State $\alpha_0 \gets 0$, $\alpha_1 > 0$ and $i\gets 0$.
        \While{$i \leq i_{\max}$}
          \State $i \leftarrow i+1$
          \If { $\phi(\alpha_i) > \phi(0) + c_1 \alpha_i \phi^{\prime}(0)$ \textrm{\bf or}  $\phi(\alpha_i) \geq \phi(\alpha_{i-1}),\ i>1$} 
            \State $\alpha_{\text{low}}=\alpha_{i-1}$ and $\alpha_{\text{hi}}=\alpha_{i}$ 
            \State {\bf break}
          \ElsIf { $|\phi^{\prime}(\alpha_i)| \leq c_2 \phi^{\prime}(0)$} 
            \Return $\alpha_i$
          \ElsIf { $|\phi^{\prime}(\alpha_i)| \geq 0$} 
            \State $\alpha_{\text{low}}=\alpha_{i}$ and $\alpha_{\text{hi}}=\alpha_{i-1}$
            \State {\bf break}
          \Else
            \State Using extrapolation find $\alpha_{i+1} > \alpha_i$
          \EndIf
        \EndWhile
        \State \textbf{Call} \textsc{ZoomingPhase}
      \end{algorithmic}
    \end{algorithm}
  \end{minipage}
  \begin{minipage}[t]{0.5\linewidth}
    \begin{algorithm}[H]
    \label{alg.zooming}
      \caption{\small \textsc{ZoomingPhase}}
      \begin{algorithmic}[1]\raggedright\small
        \While{$i \leq i_{\max}$} 
        \State $i \leftarrow i+1$
        \State Interpolate to find $\alpha_i \in (\alpha_{\text{low}}, \alpha_{\text{hi}})$
        \If { $\phi(\alpha_i) > \phi(0) + c_1 \alpha_i \phi^{\prime}(0)$ \textrm{\bf or} $\phi(\alpha_i) \geq \phi(\alpha_{\text{low}})$} 
        \State $\alpha_{\text{hi}} \leftarrow \alpha_i$
        \Else
        \If { $|\phi^{\prime}(\alpha_i)| \leq c_2 \phi^{\prime}(0)$} 
        \Return $\alpha_i$
        \ElsIf { $\phi^{\prime}(\alpha_i)(\alpha_{\text{hi}}-\alpha_{\text{low}}) \geq 0 $} 
        \State $\alpha_{\text{hi}} \leftarrow \alpha_{\text{low}}$
        \EndIf
        \State $\alpha_{\text{low}} \leftarrow \alpha_i$
        \EndIf
        \EndWhile 

        \State \textbf{return} {\bf failure}
      \end{algorithmic}
    \end{algorithm}
  \end{minipage}
\end{minipage}

\section{Figure showing the effect of separation parameter}
A typical 2D data with $K=5$ created for different separation is shown in Figure~\ref{fig:scatter}.
\begin{figure*}[htbp]
\begin{center}
	
	\subfigure[low separation]{%
		\label{fig:scatter-low}%
		\includegraphics[width=.32\textwidth]{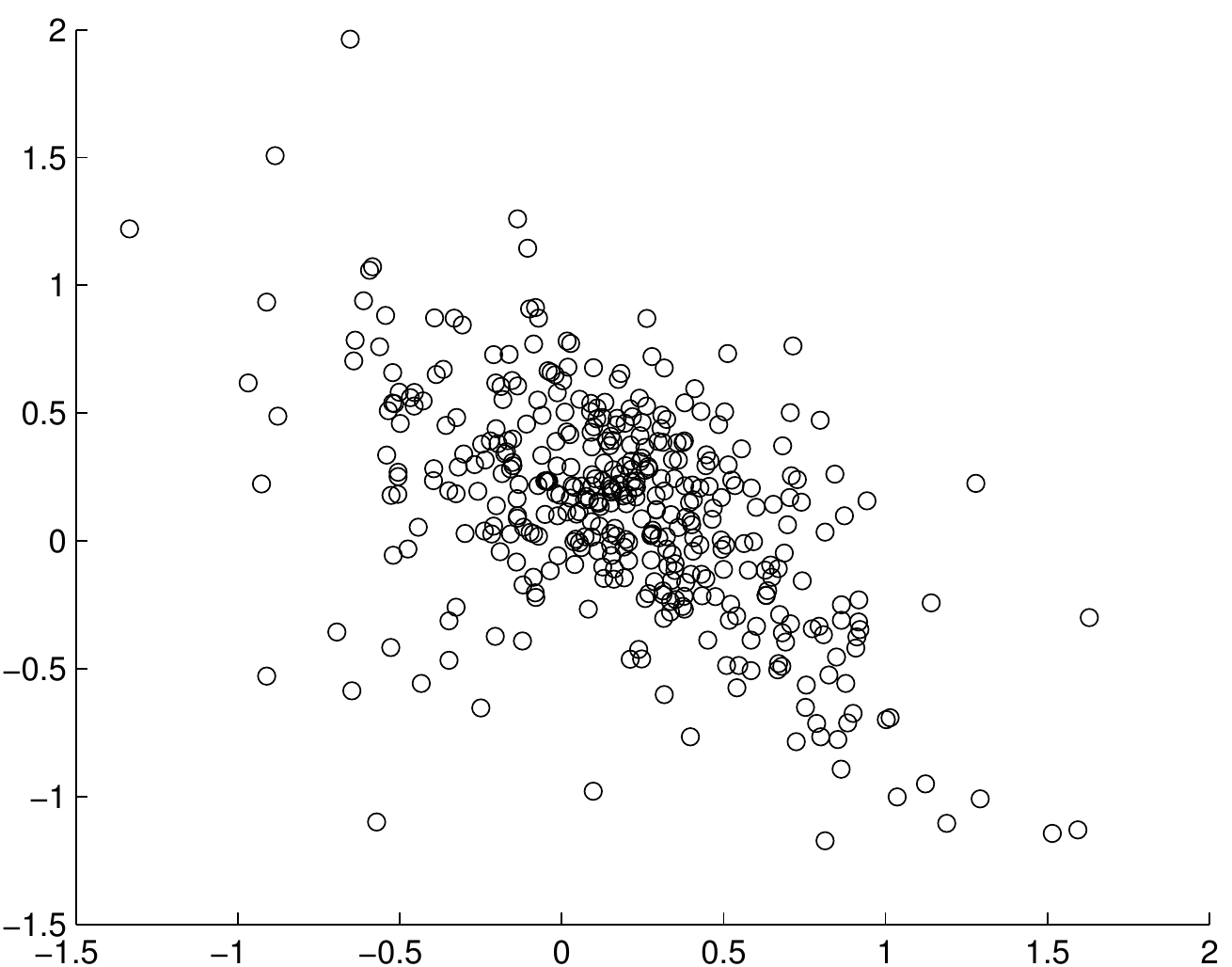}}%
	\hfill%
	\subfigure[medium separation]{%
		\label{fig:scatter-mid}%
		\includegraphics[width=.32\textwidth]{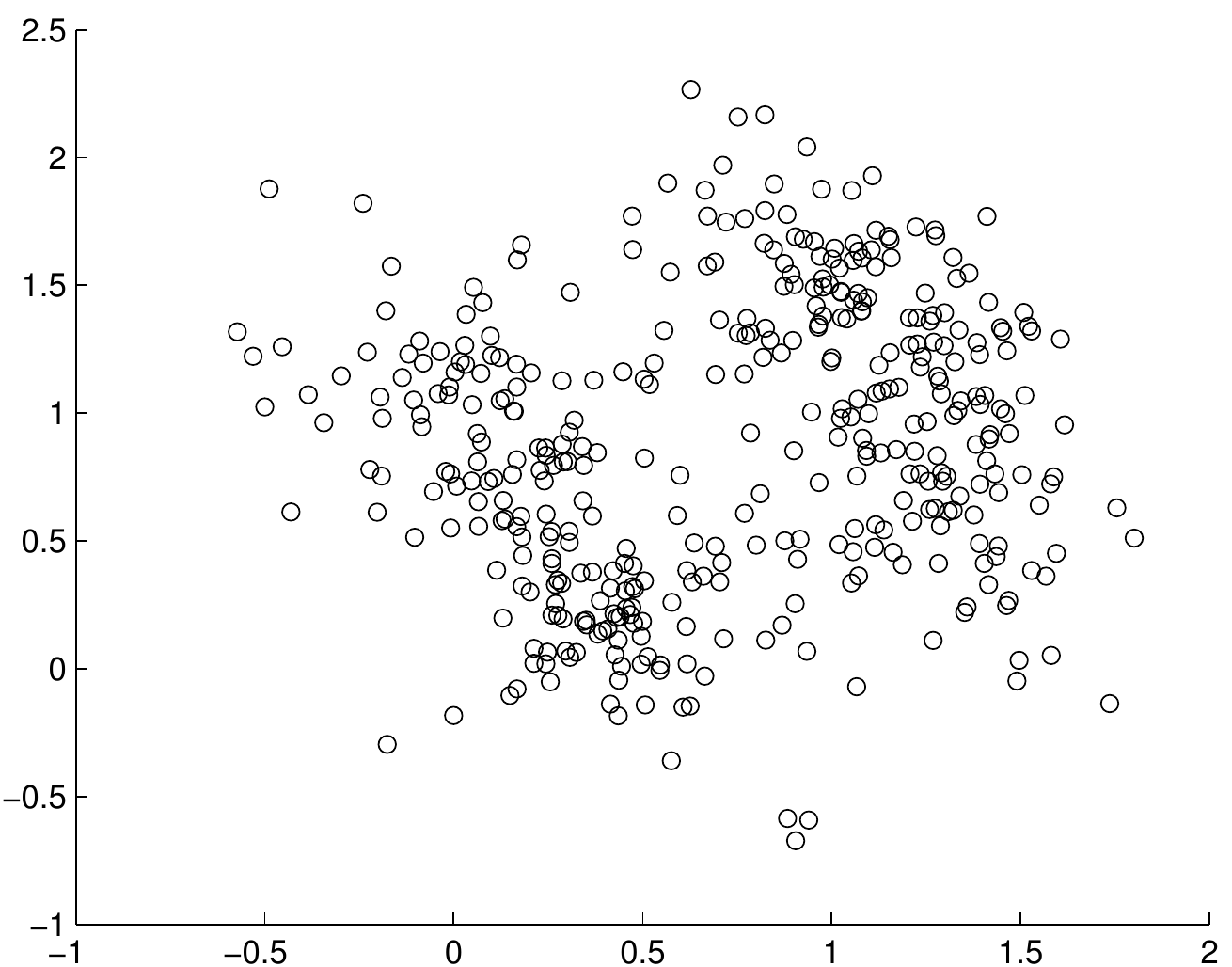}}%
	\hfill%
	\subfigure[high separation]{%
		\label{fig:scatter-high}%
		\includegraphics[width=.32\textwidth]{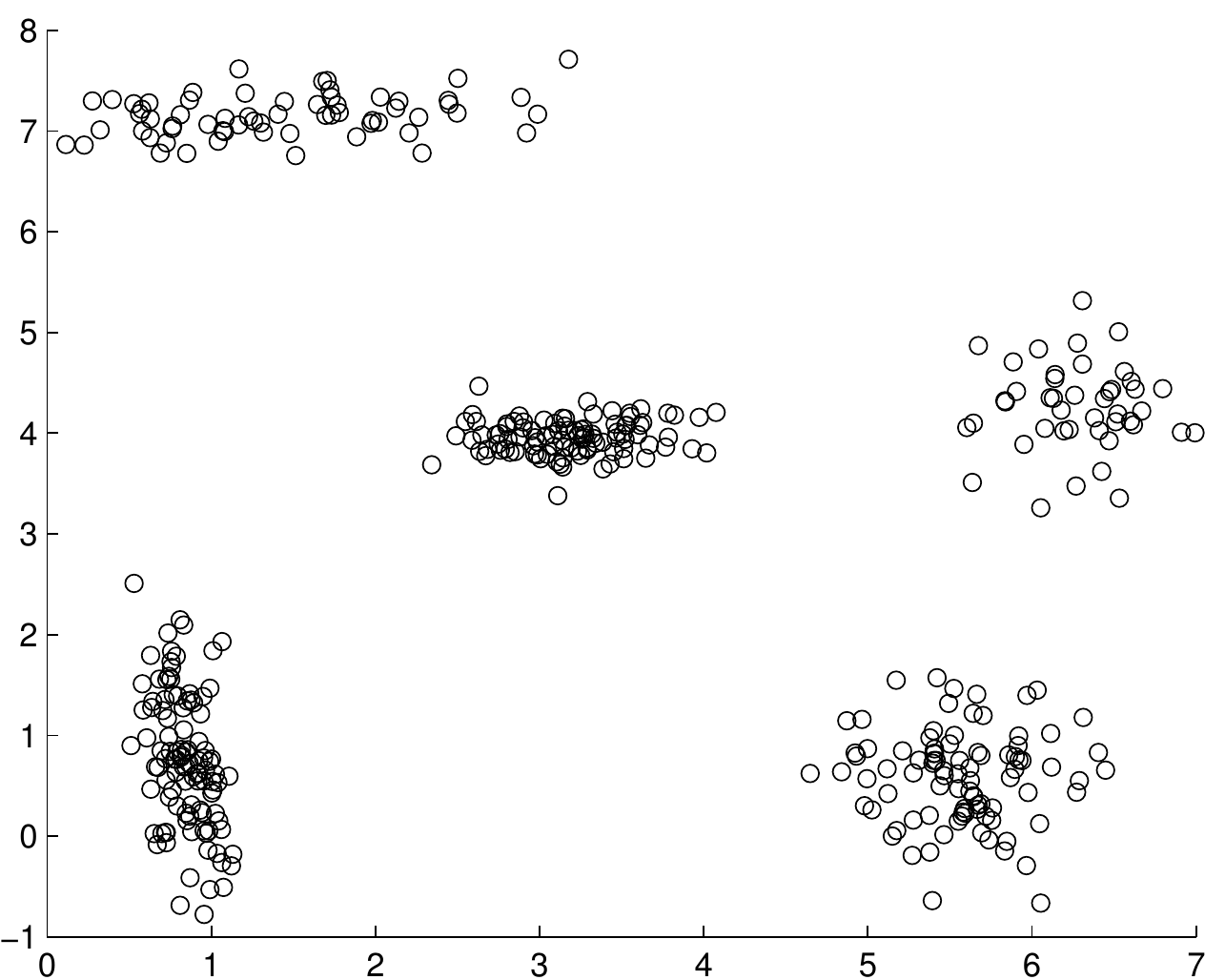}}%
	
	\caption{scatter data cloud for different degrees of separations.}
	\label{fig:scatter}
\end{center}
\end{figure*}

\section{Supplementary Simulated Experimental Results}
For the lower-dimensional cases and when the number of data is small, EM algorithm shows better performance than LBFGS optimization and pretty similar performance like CG. This is mainly because of the computational overhead like retraction and parallel transport that is needed to be computed for them. We believe that a more careful implementation will change the picture specially for the case of LBFGS. Because for performing parallel transport between $\Sigma_1$ and $\Sigma_2$, one can store the matrix $(\Sigma_2\Sigma_1^{-1})^{1/2}$ and for performing the inner product at point $\Sigma$, it is possible to store the inverse of the matrix $\Sigma$; by storing these matrices the only computation remained is matrix product. 

We reported the result for $d=20$ in the main text and because of the lack of space, we are reporting the result for usual CG below in Table~\ref{tbl:expd20usualCG}. The results for low-dimensional cases $d=2$ and $d=5$ and for pretty small number of data-points $n=d^2\times100$ are shown in tables~\ref{tbl:expd2e10}-\ref{tbl:expd5usualCG}.
\subsection{Results for $d=20$}
\begin{table}[H]
  \caption{Speed and ALL for Usual CG and with $d=20$.}
  \label{tbl:expd20usualCG}
  \rowcolors{2}{gray!20}{white}
  \begin{tabular}{lr|rr|rr}
    \hline
    & & \multicolumn{2}{c|}{$e=1$}   & \multicolumn{2}{c}{$e=10$} \\
    &   & Time (s) & ALL       & Time (s) & ALL\\
    \hline
    $c=0.2$ & $K=2$ & 57.2 $\pm$ 49.7 & 17.6 & 26.4 $\pm$ 29.2 & -11.3 \\
& $K=5$ & 225.3 $\pm$ 74.3 & 17.5 & 216.7 $\pm$ 105.3 & -12.9 \\[5pt]
    $c=1$ & $K=2$ & 43.1 $\pm$ 22.6 & 17.0 & 27.5 $\pm$ 15.8 & -10.8 \\
& $K=5$ & 191.1 $\pm$ 86.8 & 16.2 & 140.4 $\pm$ 39.7 & -13.4\\[5pt]
    $c=5$ & $K=2$ & 18.4 $\pm$ 8.9 & 17.1 & 36.4 $\pm$ 20.5 & -11.2 \\
& $K=5$ & 97.8 $\pm$ 48.1 & 16.1 & 167.4 $\pm$ 90.5 & -12.8 
  \end{tabular}
\end{table}
\subsection{Results for $d=2$}
\begin{table}[htb]
  \caption{Speed and log-likelihood comparisons for $d=2$ and $e=10$}
  \label{tbl:expd2e10}
  \rowcolors{2}{gray!20}{white}
  \begin{tabular}{lr|rr|rr|rr}
    \hline
    & & \multicolumn{2}{c|}{EM Algorithm}   & \multicolumn{2}{c|}{LBFGS Reparametrized} & \multicolumn{2}{c}{CG Reparametrized} \\
    &   & Time & ALL       & Time & ALL& Time & ALL \\
    \hline
    $c=0.2$ & $K=2$ & 0.4 $\pm$ 0.4 & 0.6 & 1.7 $\pm$ 1.0 & 0.6 & 0.6 $\pm$ 0.4 & 0.6 \\
& $K=5$ & 1.4 $\pm$ 1.0 & -0.6 & 7.3 $\pm$ 4.0 & -0.6 & 2.1 $\pm$ 2.3 & -0.6 \\[5pt]
    $c=1$ & $K=2$ & 0.4 $\pm$ 0.3 & 0.4 & 1.4 $\pm$ 0.7 & 0.4 & 0.4 $\pm$ 0.2 & 0.4 \\
& $K=5$ & 1.0 $\pm$ 1.0 & -1.3 & 4.6 $\pm$ 2.7 & -1.3 & 1.2 $\pm$ 0.8 & -1.3 \\[5pt]
    $c=5$ & $K=2$ & 0.0 $\pm$ 0.0 & 0.2 & 0.1 $\pm$ 0.1 & 0.2 & 0.1 $\pm$ 0.0 & 0.2 \\
& $K=5$ & 0.1 $\pm$ 0.1 & -2.0 & 2.0 $\pm$ 2.5 & -2.0 & 0.4 $\pm$ 0.4 & -2.0 
  \end{tabular}
\end{table}

\begin{table}[htb]
  \caption{Speed and log-likelihood comparisons for $d=2$ and $e=1$}
  \rowcolors{2}{gray!20}{white}
  \begin{tabular}{lr|rr|rr|rr}
    \hline
    & & \multicolumn{2}{c|}{EM Algorithm}   & \multicolumn{2}{c|}{LBFGS Reparametrized} & \multicolumn{2}{c}{CG Reparametrized} \\
    &   & Time & ALL       & Time & ALL& Time & ALL \\
    \hline
    $c=0.2$ & $K=2$ & 0.7 $\pm$ 0.6 & 1.8 & 1.4 $\pm$ 0.8 & 1.8 & 0.7 $\pm$ 0.4 & 1.8 \\
    & $K=5$ & 2.5 $\pm$ 1.8 & 1.8 & 5.5 $\pm$ 1.7 & 1.8 & 2.4 $\pm$ 0.9 & 1.8 \\[5pt]
    $c=1$ & $K=2$ & 0.7 $\pm$ 0.5 & 1.6 & 1.7 $\pm$ 0.9 & 1.6 & 0.8 $\pm$ 0.5 & 1.6 \\
    & $K=5$ & 2.1 $\pm$ 1.1 & 1.1 & 5.1 $\pm$ 1.8 & 1.1 & 2.8 $\pm$ 1.1 & 1.1 \\[5pt]
    $c=5$ & $K=2$ & 0.0 $\pm$ 0.1 & 1.1 & 0.3 $\pm$ 0.3 & 1.1 & 0.1 $\pm$ 0.1 & 1.1 \\
    & $K=5$ & 0.3 $\pm$ 0.4 & 0.2 & 1.8 $\pm$ 1.3 & 0.2 & 0.9 $\pm$ 0.6 & 0.2 
  \end{tabular}
\end{table}

\begin{table}[htb]
  \caption{Speed and ALL for Usual CG and with $d=2$.}
  \label{tbl:expd2usualCG}
  \rowcolors{2}{gray!20}{white}
  \begin{tabular}{lr|rr|rr}
    \hline
    & & \multicolumn{2}{c|}{$e=1$}   & \multicolumn{2}{c}{$e=10$} \\
    &   & Time (s) & ALL       & Time (s) & ALL\\
    \hline
    $c=0.2$ & $K=2$ & 1.0 $\pm$ 0.5 & 1.8 & 0.9 $\pm$ 0.4 & 0.6 \\
& $K=5$ & 5.3 $\pm$ 2.0 & 1.8 & 6.7 $\pm$ 3.9 & -0.6 \\[5pt]
    $c=1$ & $K=2$ & 1.0 $\pm$ 0.4 & 1.6 & 0.8 $\pm$ 0.4 & 0.4 \\
& $K=5$ & 7.8 $\pm$ 4.8 & 1.1 & 3.9 $\pm$ 1.8 & -1.3\\[5pt]
    $c=5$ & $K=2$ & 0.7 $\pm$ 0.7 & 1.1 & 0.3 $\pm$ 0.1 & 0.2 \\
& $K=5$ & 4.8 $\pm$ 5.2 & 0.2 & 3.2 $\pm$ 2.3 & -2.0 
  \end{tabular}
\end{table}

\clearpage
\subsection{Results for $d=5$}

\begin{table}[htb]
  \caption{Speed and log-likelihood comparisons for $d=5$ and $e=10$}
  \rowcolors{2}{gray!20}{white}
  \begin{tabular}{lr|rr|rr|rr}
    \hline
    & & \multicolumn{2}{c|}{EM Algorithm}   & \multicolumn{2}{c|}{LBFGS Reparametrized} & \multicolumn{2}{c}{CG Reparametrized} \\
    &   & Time & ALL       & Time & ALL& Time & ALL \\
    \hline
    $c=0.2$ & $K=2$ & 0.1 $\pm$ 0.0 & -1.4 & 0.8 $\pm$ 0.6 & -1.4 & 0.2 $\pm$ 0.1 & -1.4 \\
& $K=5$ & 3.1 $\pm$ 2.7 & -3.1 & 14.0 $\pm$ 12.0 & -3.1 & 3.6 $\pm$ 2.0 & -3.1 \\[5pt]
    $c=1$ & $K=2$ & 0.1 $\pm$ 0.0 & -0.7 & 0.7 $\pm$ 1.0 & -0.7 & 0.2 $\pm$ 0.2 & -0.7 \\
& $K=5$ & 0.7 $\pm$ 0.5 & -3.5 & 7.0 $\pm$ 4.8 & -3.5 & 1.7 $\pm$ 1.1 & -3.5 \\[5pt]
    $c=5$ & $K=2$ & 0.0 $\pm$ 0.0 & -1.1 & 0.2 $\pm$ 0.2 & -1.1 & 0.1 $\pm$ 0.1 & -1.1 \\
& $K=5$ & 0.9 $\pm$ 1.1 & -3.7 & 4.9 $\pm$ 5.1 & -3.7 & 1.4 $\pm$ 1.2 & -3.7 \\
  \end{tabular}
\end{table}


\begin{table}[htb]
  \caption{Speed and log-likelihood comparisons for $d=5$ and $e=1$}
  \rowcolors{2}{gray!20}{white}
  \begin{tabular}{lr|rr|rr|rr}
    \hline
    & & \multicolumn{2}{c|}{EM Algorithm}   & \multicolumn{2}{c|}{LBFGS Reparametrized} & \multicolumn{2}{c}{CG Reparametrized} \\
    &   & time & ALL       & time & ALL& time & ALL \\
    \hline
    $c=0.2$ & $K=2$ & 1.9 $\pm$ 2.0 & 4.4 & 3.6 $\pm$ 1.5 & 4.4 & 1.8 $\pm$ 1.1 & 4.4 \\
    & $K=5$ & 4.3 $\pm$ 1.7 & 4.4 & 13.9 $\pm$ 4.8 & 4.4 & 6.7 $\pm$ 2.4 & 4.4 \\[5pt]
    $c=1$ & $K=2$ & 1.3 $\pm$ 1.2 & 4.1 & 2.1 $\pm$ 1.2 & 4.0 & 1.1 $\pm$ 0.7 & 4.1 \\
    & $K=5$ & 3.3 $\pm$ 2.2 & 3.5 & 9.5 $\pm$ 7.3 & 3.5 & 5.2 $\pm$ 2.1 & 3.5 \\[5pt]
    $c=5$ & $K=2$ & 0.0 $\pm$ 0.0 & 3.8 & 0.2 $\pm$ 0.2 & 3.8 & 0.2 $\pm$ 0.1 & 3.8 \\
    & $K=5$ & 0.7 $\pm$ 1.5 & 2.8 & 3.3 $\pm$ 3.5 & 2.8 & 1.7 $\pm$ 1.9 & 2.8
  \end{tabular}
\end{table}

\begin{table}[htb]
  \caption{Speed and ALL for Usual CG and with $d=5$.}
  \label{tbl:expd5usualCG}
  \rowcolors{2}{gray!20}{white}
  \begin{tabular}{lr|rr|rr}
    \hline
    & & \multicolumn{2}{c|}{$e=1$}   & \multicolumn{2}{c}{$e=10$} \\
    &   & Time (s) & ALL       & Time (s) & ALL\\
    \hline
    $c=0.2$ & $K=2$ &1.9 $\pm$ 0.8 & 4.4 & 0.9 $\pm$ 0.2 & -1.4 \\
& $K=5$ & 12.2 $\pm$ 5.5 & 4.4& 10.6 $\pm$ 6.7 & -3.1 \\[5pt]
$c=1$ & $K=2$ & 3.2 $\pm$ 2.1 & 4.0 & 1.3 $\pm$ 0.7 & -0.7 \\
& $K=5$ &11.2 $\pm$ 4.9 & 3.5 & 7.4 $\pm$ 3.0 & -3.5 \\[5pt]
$c=5$ & $K=2$ & 0.9 $\pm$ 0.4 & 3.8 & 0.7 $\pm$ 0.3 & -1.1 \\
& $K=5$ & 5.8 $\pm$ 3.5 & 2.8 & 6.1 $\pm$ 6.0 & -3.7
  \end{tabular}
\end{table}

\section{Supplementary experimental results on some real datasets}
We selected some datasets  from UCI machine learning dataset repository\footnote{https://archive.ics.uci.edu/ml/datasets} and report the results for all of those we selected to perform the test on. As it can be seen from performance evaluations for these real datasets and also other simulated and real datasets in the main text, a systematic behavior for different optimization procedures can be observed. Namely by increasing the number of components, overlap increases leading to inferior performance of EM in compare to manifold optimization methods. For two of dataset, we normalized the features to have equal variance due to high variability of feature variances.
\subsection{Results for MAGIC gamma telescope}
In the case of MAGIC telescope dataset, the reparametrization proves to be extremely important, such that the stopping criterion of small cost difference is triggered without the algorithm being actually converged. 
\begin{table}[htb]
  \caption{Speed and ALL comparisons  for MAGIC gamma telescope data $d=10$, $n=19020$}
  \label{tbl:expmagic}
  \begin{tabular}{lrr|rr|rr|rr}
    \hline
    & \multicolumn{2}{c|}{EM Algorithm}   & \multicolumn{2}{c|}{LBFGS Reparam} & \multicolumn{2}{c|}{CG Reparam} & \multicolumn{2}{c}{CG Usual} \\
    & Time (s) & ALL       & Time (s) & ALL& Time (s) & ALL & Time (s) & ALL \\
    \hline
$K=2$ & 0.28 & -28.44 & 1.08 & -28.44 & 0.54 & -28.44 & 33.57 & -29.47 \\
$K=3$ & 1.10 & -27.60 & 4.12 & -27.56 & 2.74 & -27.56 & 127.98 & -29.14 \\
$K=4$ & 3.59 & -27.29 & 4.07 & -27.29 & 2.14 & -27.29 & 125.78 & -28.62 \\
$K=5$ & 3.16 & -27.03 & 7.40 & -27.03 & 10.14 & -27.03 & 222.54 & -28.75 \\
$K=6$ & 10.16 & -26.90 & 9.89 & -26.92 & 8.86 & -26.92 & 304.88 & -28.09 \\
$K=7$ & 10.38 & -26.79 & 11.02 & -26.87 & 18.00 & -26.75 & 395.92 & -27.99 \\
$K=8$ & 9.01 & -26.64 & 14.97 & -26.63 & 16.04 & -26.64 & 448.02 & -27.62 \\
$K=9$ & 27.89 & -26.63 & 18.74 & -26.66 & 17.91 & -26.66 & 505.05 & -27.66 \\
$K=10$ & 16.16 & -26.47 & 18.21 & -26.49 & 22.23 & -26.49 & 552.52 & -27.70 \\
    \hline 
  \end{tabular}
\end{table}
\subsection{Results for (normalized) Corel image features}
\begin{table}[htb]
 \caption{Speed and ALL comparisons  for Corel image features data $d=57$, $n=68040$}
  \label{tbl:expcorel}
  \begin{tabular}{lrr|rr|rr}
    \hline
    & \multicolumn{2}{c|}{EM Algorithm}   & \multicolumn{2}{c|}{LBFGS Reparam} & \multicolumn{2}{c}{CG Reparam}  \\
    & Time (s) & ALL       & Time (s) & ALL& Time (s) & ALL  \\
    \hline
$K=2$ & 13.63 & -13.34 & 18.18 & -13.34 & 20.42 & -13.34 \\
$K=3$ & 133.59 & -4.78 & 164.52 & -4.78 & 114.07 & -4.79 \\
$K=4$ & 64.56 & 0.26 & 96.13 & 0.26 & 70.15 & 0.26 \\
$K=5$ & 178.76 & 3.22 & 110.39 & 3.22 & 91.87 & 3.20 \\
$K=6$ & 465.93 & 4.53 & 300.52 & 5.25 & 361.56 & 5.24 \\
$K=7$ & 646.85 & 7.02 & 347.00 & 7.03 & 712.65 & 6.85 \\
$K=8$ & 1124.44 & 8.62 & 442.05 & 8.59 & 557.63 & 8.49 \\
$K=9$ & 913.35 & 9.84 & 1163.63 & 10.09 & 981.04 & 9.80 \\
$K=10$ & 2213.15 & 10.81 & 592.88 & 10.79 &1456.38 & 10.79 \\
    \hline 
  \end{tabular}
\end{table}
\newpage
\subsection{Results for combined cycle power plant}
\begin{table}[htb]
 \caption{Speed and ALL comparisons  for power plant data $d=5$, $n=2568$}
  \label{tbl:expplant0}
  \begin{tabular}{lrr|rr|rr}
    \hline
    & \multicolumn{2}{c|}{EM Algorithm}   & \multicolumn{2}{c|}{LBFGS Reparam} & \multicolumn{2}{c}{CG Reparam}  \\
    & Time (s) & ALL       & Time (s) & ALL& Time (s) & ALL  \\
    \hline
$K=2$ & 0.14 & -16.09 & 0.31 & -16.09 & 0.21 & -16.09 \\
$K=3$ & 1.41 & -15.99 & 3.99 & -15.98 & 1.82 & -15.98 \\
$K=4$ & 1.94 & -15.91 & 4.56 & -15.91 & 1.99 & -15.91 \\
$K=5$ & 2.50 & -15.87 & 3.40 & -15.88 & 2.13 & -15.88 \\
$K=6$ & 3.79 & -15.83 & 7.56 & -15.82 & 4.78 & -15.82 \\
$K=7$ & 9.18 & -15.81 & 7.39 & -15.80 & 3.58 & -15.80 \\
$K=8$ & 12.44 & -15.78 & 17.04 & -15.74 & 9.32 & -15.74 \\
$K=9$ & 11.41 & -15.76 & 17.39 & -15.76 & 36.41 & -15.76 \\
$K=10$ & 73.27 & -15.69 & 52.41 & -15.69 & 23.06 & -15.69 \\
    \hline 
  \end{tabular}
\end{table}

\subsection{Results for (normalized) YearPredictionMSD}
\begin{table}[htb]
 \caption{Speed and ALL comparisons  for YearPredictionMSD data $d=90$, $n=515345$}
  \label{tbl:expplant2}
  \begin{tabular}{lrr|rr|rr}
    \hline
    & \multicolumn{2}{c|}{EM Algorithm}   & \multicolumn{2}{c|}{LBFGS Reparam} & \multicolumn{2}{c}{CG Reparam}  \\
    & Time (s) & ALL       & Time (s) & ALL& Time (s) & ALL  \\
    \hline
$K=2$ & 248.14 & -86.67 & 224.73 & -86.67 & 196.11 & -86.67 \\
$K=3$ & 352.74 & -82.00 & 549.42 & -82.00 & 752.15 & -82.00 \\
$K=4$ & 816.22 & -79.79 & 1212.66 & -79.79 & 1832.93 & -79.79 \\
$K=5$ & 5152.93 & -78.13 & 5959.86 & -78.13 & 3061.53 & -80.02 \\
$K=6$ & 2921.52 & -76.96 & 1415.24 & -76.96 & 3084.32 & -76.96 \\
$K=7$ & 4717.05 & -76.09 & 4690.40 & -76.09 & 5813.55 & -76.09 \\
$K=8$ & 5528.35 & -75.32 & 3466.55 & -75.32 & 4518.16 & -75.32 \\
$K=9$ & 10729.09 & -74.76 & 5015.60 & -74.76 & 8703.81 & -74.76 \\
    \hline 
  \end{tabular}
\end{table}

\clearpage
\section{Pseucode for Riemannian LBFGS}
\begin{algorithm}[ht]
  \caption{\small L-RBFGS}
  \label{alg.lrbfgs}
  \begin{algorithmic}
    \State {\bf Given:} Riemannian manifold $\Mc$ with Riemannian metric $g$; parallel transport $\mathcal{T}$ on $\Mc$; geodesics $R$; initial value $X_0$; a smooth function $f$
    \State Set initial $H_{\rm diag}=1/\sqrt{g_{X_0}(\text{grad} f(X_0),\text{grad} f(X_0))}$ 
    \For{$k=0,1,\ldots$}
    \State Obtain descent direction $\xi_k$ by unrolling the RBFGS method\\
    \hskip16pt$\xi_k \gets \textsc{HessMul}(-\text{grad} f(X_k), k)$
    \State Use line-search to find $\alpha$ such that it satisfies Wolfe conditions
    \State Calculate $X_{k+1}=R_{X_k}(\alpha \xi_k)$
    \State Define $S_k=\mathcal{T}_{X_k,X_{k+1}}(\alpha \xi_k)$
    \State Define $Y_k=\text{grad} f(X_{k+1})-\mathcal{T}_{X_k,X_{k+1}}(\text{grad} f(X_k))$
    \State Update $H_{\text{diag}}=g_{X_{k+1}}(S_k,Y_k)/g_{X_{k+1}}(Y_k,Y_k)$
    \State Store $Y_k$; $S_k$; $g_{X_{k+1}}(S_k,Y_k)$; $g_{X_{k+1}}(S_k,S_k)/g_{X_{k+1}}(S_k,Y_k)$; $H_{\rm diag}$
    \EndFor
    \State \Return $X_k$
    \State \textbf{function} $\textsc{HessMul}(P, k)$
    \If{$k>0$}
      \State  $P_{k}=P-\frac{g_{X_{k+1}}(S_k,P_{k+1})}{g_{X_{k+1}}(Y_k,S_k)}Y_k$
      \State   $\hat{P}=\mathcal{T}_{X_{k+1},X_k}\textsc{HessMul}(\mathcal{T}_{X_k,X_{k+1}} P_{k},k-1)$
      \Return $\hat{P} - \frac{g_{X_{k+1}}(Y_k,\hat{P})}{g_{X_{k+1}}(Y_k,S_k)}S_k + \frac{g_{X_{k+1}}(S_k,S_k)}{g_{X_{k+1}}(Y_k,S_k)}P$
      \Else
      \State \Return $H_{\rm diag}P$
     \EndIf
    \State \textbf{end function} 
  \end{algorithmic}
\end{algorithm}

\end{appendices}

\end{document}